\documentclass[final]{elsarticle}

\makeatletter
\def\ps@pprintTitle{%
  \let\@oddhead\@empty
  \let\@evenhead\@empty
  \def\@oddfoot{\centerline{\thepage}}
  \let\@evenfoot\@oddfoot}
\makeatother
\usepackage{microtype}
\usepackage{graphicx}
\usepackage{subfigure}
\usepackage{booktabs} % for professional tables
\usepackage{epstopdf}
\usepackage{color}
\usepackage{amssymb,amsmath,amsthm, amsfonts}
\usepackage{hyperref}
\usepackage{tabu,multirow}
\usepackage{array}
\usepackage{pgffor}
\RequirePackage{algorithm}
\RequirePackage{algorithmic}
\RequirePackage{natbib}

\graphicspath{{figure/}}
\DeclareGraphicsExtensions{.eps}

\def \E {\mathop{{\mathbb{E}}}\limits}

\def \N {\mathcal{N}}
\def \L {\mathcal{L}}

\def \R {\mathbb{R}}

\def \zb {\mathbf{z}}
\def \xb {\mathbf{x}}
\def \yb {\mathbf{y}}
\def \Zb {\mathbf{Z}}
\def \Xb {\mathbf{X}}
\def \Yb {\mathbf{Y}}
\def \encb {\mathbf{enc}}
\def \decb {\mathbf{dec}}
\def \preb {\mathbf{pre}}

\newtheorem{definition}{ Definition}   %%[section]
\newtheorem{theorem}{ Theorem}   %%[section]
\newtheorem{corollary}{ Corollary}

\usepackage{times}  %Required
\usepackage{helvet}  %Required
\usepackage{courier}  %Required
\usepackage{url}  %Required
\usepackage{graphicx}  %Required
\usepackage{hyperref}
\begin{document}

\begin{frontmatter}

\title{Discovering Influential Factors in Variational Autoencoders}
%\tnotetext[mytitlenote]{Fully documented templates are available in the elsarticle package on \href{http://www.ctan.org/tex-archive/macros/latex/contrib/elsarticle}{CTAN}.}

%% Group authors per affiliation:
\author[xi]{Shiqi Liu\fnref{myfootnote}}
\author[br]{Jingxin Liu\fnref{myfootnote}}
\author[xi]{Qian Zhao}
\author[xi]{Xiangyong Cao}
\author[xi]{ Huibin Li}
\author[xi]{ Deyu Meng}%\corref{mycorrespondingauthor}}
%\ead{deyumeng@mail.xjtu.edu.cn}
\author[br]{Hongying Meng}
\author[ny]{Sheng Liu}
\fntext[myfootnote]{Shiqi Liu and Jingxin Liu made equal contributions to this work.}
%% or include affiliations in footnotes:
%\cortext[mycorrespondingauthor]{Corresponding author}
\address[xi]{Xi'an Jiaotong University, Xi'an, Shaan'xi Province, P. R. China}
\address[br]{Brunel University London, London, United Kingdom}
\address[ny]{State University of New York at Buffalo, NY 14214, United States}

\begin{abstract}
In the field of machine learning, it is still a critical issue to identify and supervise the learned representation without manually intervening or intuition assistance to extract useful knowledge or serve for the downstream tasks. In this work, we focus on supervising the influential factors extracted by the variational autoencoder(VAE). The VAE is proposed to learn independent low dimension representation while facing the problem that sometimes pre-set factors are ignored. We argue that the mutual information of the input and each learned factor of the representation plays a necessary indicator of discovering the influential factors. We find the VAE objective inclines to induce mutual information sparsity in factor dimension over the data intrinsic dimension and therefore result in some non-influential factors whose function on data reconstruction could be ignored. We show mutual information also influences the lower bound of the VAE's reconstruction error and downstream classification task. To make such indicator applicable, we design an algorithm for calculating the mutual information for the VAE and prove its consistency. Experimental results on MNIST, CelebA and DEAP datasets show that mutual information can help determine influential factors, of which some are interpretable and can be used to further generation and classification tasks, and help discover the variant that connects with emotion on DEAP dataset.\end{abstract}


\end{frontmatter}
\newpage

\section{Introduction}
Learning efficient low dimension representation of data is important in machine learning and related applications. Efficient and intrinsic low dimension representation is helpful to exploit the underlying knowledge of data and serves for latter tasks including generation, classification and association. Early linear dimension reduction (Principle Component Analysis (\cite{bishop2006pattern},\cite{zhao2014robust})) has been widely used in primary data analysis and its variant has been applied in face identification  (\cite{yang2004two}) and classical linear independent representation (Independent Component Analysis  (\cite{hyvarinen2004independent},\cite{hyvarinen1999nonlinear}) have been used in  blind source separation (\cite{jung2000removing}) and EEG signal processing   (\cite{makeig1996independent}). Nonlinear dimension reduction (e.g. Autoencoder  \cite{goodfellow2016deep},\cite{vincent2010stacked},\cite{fan2019autoencoder}) begins to further learn abstract representation\cite{bhatt2019representation},\cite{zhang2019depth} and has been used in semantic hashing  (\cite{salakhutdinov2009semantic} and many other tasks(\cite{lore2017llnet},\cite{liu2019stacked},\cite{hou2019sparse}). Recently, a new technique, called variational autoencoder (\cite{kingma2013auto,rezende2014stochastic}) has attracted much attention of researchers, due to its capability in extracting nonlinear independent representation. The method can further  model causal relationship, represent disentangled visual  variants (\cite{mathieu2016disentangling},\cite{larsen2015autoencoding},\cite{higgins2016early})\footnote{According to \cite{locatello2018challenging}, although the method proves effective at making the disentangled factors not correlated, the learned disentangled factors are still correlated with each other.}
 and interpretable time series variants  (\cite{hsu2017unsupervised},\cite{denton2017unsupervised}) and this method can serve for
generating signals with abundant diversities
in a ``factor-controllable" way  (\cite{suzuki2016joint},\cite{higgins2017scan},\cite{hu2017toward}). The related techniques enable the knowledge transferring through
shared factors among different tasks (\cite{higgins2017darla}).

However, the usage of the VAE on extracting factors are unclear and we lack efficient methodologies to quantify the influence of each learned factor on data representation. In application,  sometimes some  pre-set factors remain unused \footnote{ Montage (D) in Fig.(\ref{fig:Iencoderzhsigma2_argumented}) is a typical traversal of the unused factor. } (\cite{goodfellow2016deep},\cite{van2017neural}), and the relation between the learned factors and original data has to be discovered by manually  intervention (visual or  aural observation). This leads to the waste on extra factors and hinders the factor selection for the subsequent tasks such as generating meaningful image/audio. Besides, some classical influence determination methods including estimating the variance of each factor lose its utility on the VAE. Therefore,  identifying and monitoring the influential factor of the VAE  becomes a critical issue along this line of research.

\begin{figure*}[t]
  \centering
  \includegraphics[width=\linewidth,clip=true]{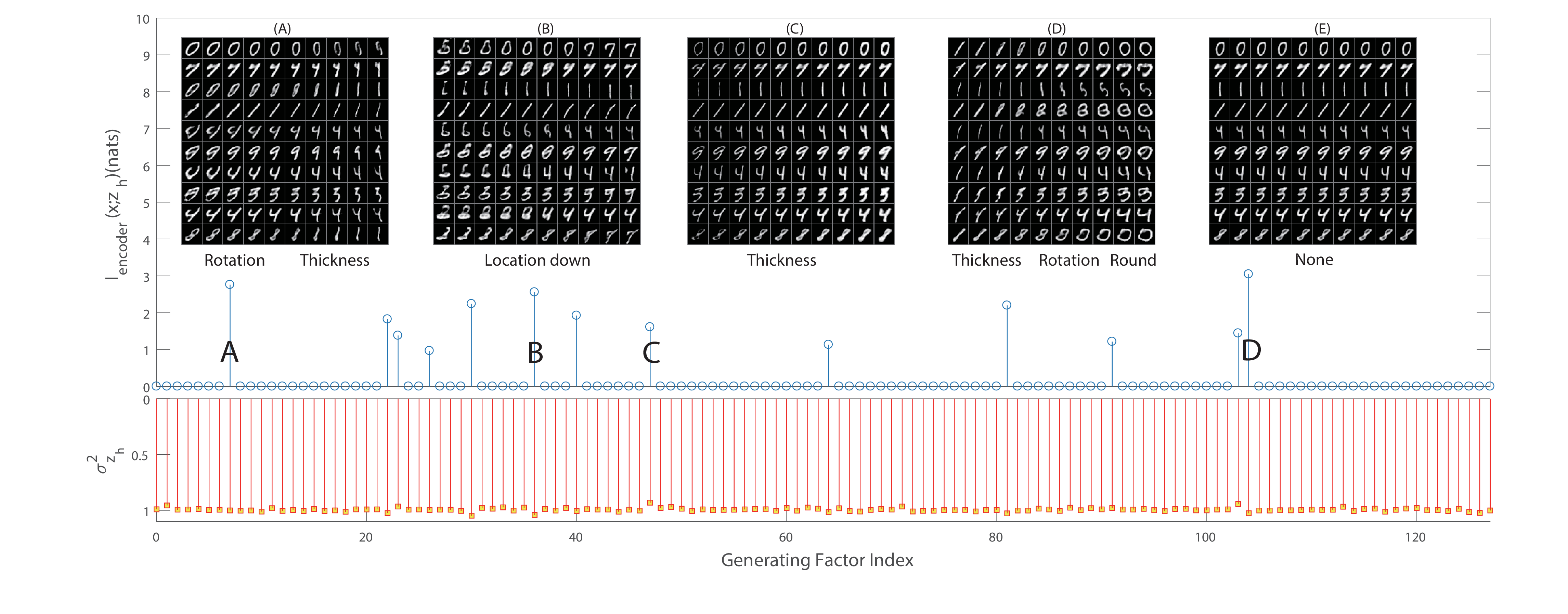}
  \caption{\textbf{Estimated $I(\Xb;{Z_{\encb}}_h)$ determines the influential factors; } $I(\Xb;{Z_{\encb}}_h)$, $\sigma^2_{z_h}$ and qualitatively influential factor traversals of $\beta(=10)$-VAE on MNIST.
  The top pulse subgraph:  $I(\Xb;{Z_{\encb}}_h)$ of each factor.
  The bottom reverse pulse subgraph:  the estimated variance $\sigma^2_{z_h}$ of each factor.
  The A,B,C montages: influential factor traversals corresponding to factor A,B,C noted in the pulse graph and  the whole influential factor traversals are listed in Fig.(\ref{fig:traservel of beta10MNIST}) in Appendix~\ref{Experiment Details}. The montages D is the traversal of ignored factors with little estimated mutual information. According to the four montages, the variance can't determine the influential factors as mutual information indicator does. }\label{fig:Iencoderzhsigma2_argumented}
\end{figure*}

In order to efficiently  determine and supervise the learned factors, this paper has made the following efforts.

\begin{itemize}
  \item We first adopt mutual information as the quantitative indicator of assessing the influence of each factor on data representation in the VAE model. Besides, in order to analyze the rationality of this indicator, we theoretically prove that how mutual information influence the lower bound of the VAE's reconstruction error and subsequent classification task.
  \item We propose an estimation algorithms to calculate the mutual information for all the factors of the VAE, and then we prove its consistency.
  \item  We substantiate the effectiveness of the proposed indicator by  experiments on MNIST  (\cite{L1998Gradient}), CelebA  (\cite{liu2015deep})  and DEAP  (\cite{koelstra2012DEAP}). Especially, some discovered factors by the proposed indicator are found meaningful and interpretable for data representation  and other left ones are generally ignorable for the task. The capability of the selected factors on generalization and classification tasks  are also verified.
      \end{itemize}

This paper is organized as the following. We introduce the VAE model for generation and classification in Section 2. We argue the necessity of mutual information as a indicator in Section 3. Specifically, we introduce the mutual information of input data and factors, analyze the cause through the perspective of mutual information and data intrinsic dimension, discuss the relationship of  mutual information and recover as well as the classification and propose the estimator and prove its consistency. We review the related work on supervising the factors of the VAE in Section 4. The experiments are in Section 5.

Throughout the paper, we denote a random variable in upper case, e.g., $Z$; a random vector in bold upper case, e.g., $\Zb$, whose $h^{th}$ component is denoted as $Z_h$; a general vector as bold lower case, e.g., $\zb$. More notations appearing in the following contents are listed in Table \ref{tab:Margin_settings} for easy reference.

    \begin{table}[h]
		\centering
		\begin{tabular}{ l l }\hline
			Notation &Explanation\\\hline
$\Xb$ & random variables representing the data\\\hline
$\Yb$ & random variables representing the data label\\\hline
$\Zb$ & random variables with $p_{\decb}(\zb)=\N(\zb|\textbf{0},I_H)$\\\hline
$ \Zb_{\encb}$& random variables with $q_\encb(\zb)=\int_{}^{}q_\encb(\zb|\xb)p_{data}(\xb)d\xb$\\\hline
${Z_{\encb}}_h$& a random variable with index $h$ of $ \Zb_{\encb}$\\\hline
$ \Yb_{\preb}$& random variables with $p_\preb(\yb)=\int_{}^{}p_\preb(\yb|\zb)q_\encb(\zb|\xb) p_{data}(\xb)d\xb$\\\hline
$\hat \Xb(\Zb_{\encb})$& a function named $\hat \Xb$ of $\Zb_{\encb}$, the  estimator of $\Xb$   \\\hline
$H(\Xb)$ &  $\E_{\xb\sim p(\xb)}-\log p(\xb)$\\\hline
$I(\Xb;\Zb)$ &   $\E_{\xb\sim p(\xb)}D_{KL}(p(\zb|\xb)||p(\zb))$
 \\\hline
$\Xb_{rec}$ & $\decb_\mu(\Zb_{\encb})$ where $\decb_\mu$ is such that $ p_\decb(\xb|\zb)=\N(\xb|\decb_\mu(\zb),\decb_\sigma)$\\\hline			
$\Zb_{major}$& a major set of $\Zb_{\encb}$ where $\Zb_{\encb}= [\Zb_{major},\Zb_{minor}]$\\\hline
$\Zb_{minor}$& a minor set of $\Zb_{\encb}$ where $\Zb_{\encb}= [\Zb_{major},\Zb_{minor}]$\\\hline
$ \Xb_{recc}$&$\decb_\mu(\Zb_{major}, \mathbf{0})$\\\hline
$ \hat \Yb$& a random variable, the estimator of $\Yb$\\\hline

		\end{tabular}
		\caption{Notation Table}
		\label{tab:Margin_settings}
	\end{table}
\section{VAE model}
The VAE (\cite{kingma2013auto}, \cite{rezende2014stochastic}) is a scalable unsupervised representation learning model (\cite{higgins2016beta}):
the VAE assumes that input $\Xb$ is generated by several independent Gaussian random variables $\Zb$, that is $p_{\decb}(\zb)=\N(\zb|\textbf{0},I_H)$. Since Gaussian distribution can be continuously and reversibly mapping to many other distributions, the theoretical analysis on it might be also instructive for other continuous-latent VAEs.
The generating/decoding   process is modeled as $p_{\decb}(\xb|\zb)$ and the inference/encoding process $q_\encb(\zb|\xb)=\N(\zb |${\boldmath$\mu$}$(\xb),diag(\sigma_1(\xb),\cdots,\sigma_H(\xb)))$ is treated as the approximate posterior distribution. Note that it yields that  $q_\encb(\zb|\xb)=q_\encb(z_1|\xb)\cdots q_\encb(z_H|\xb)$. We assume both of them are parameterized by the neural network with parameter $\encb$ and $\decb$.

\textbf{Factor:} Let $\Zb_{\encb}$ denote random variables with $q_\encb(\zb)=\int_{}^{}q_\encb(\zb|\xb)p_{data}(\xb)d\xb$
and a factor in the latter literature refers to a dimension  of $\Zb_{\encb}$.
\subsection{Generation}
In the VAE setting, the approximate inference method is applied to
 maximizing the variational lower bound   of $\log p_{\decb}(\xb)=\log\int p_{\decb}(\xb|\zb) p_{\decb}(\zb)d\zb  $,
\begin{eqnarray}\label{Approximate Inference}
\L_{rec}&=&\E\limits_{\zb\sim q_\encb(\zb|\xb)}\log p_{\decb}(\xb|\zb)-\nonumber \\
&& D_{KL}(q_\encb(\zb|\xb)||p_{\decb}(\zb))\nonumber
\\ &\le& \log p_{\decb}(\xb),
  \end{eqnarray}
    with the equality holds iff
    \begin{equation}\label{Equality Condition}
D_{KL}(q_\encb(\zb|\xb)||p_{\decb}(\zb|\xb))=0.
  \end{equation}

In order to limit the information channel capacity (\cite{higgins2016beta}), $\beta$-VAE introduces $\beta>1$ to the second term of the objective,
\begin{eqnarray}\label{beta-Approximate Inference}
\L_{rec-\beta}&=&\E\limits_{\zb\sim q_\encb(\zb|\xb)}\log p_{\decb}(\zb|\xb)- \nonumber\\ &&\ \ \ \ \ \ \ \ \ \ \beta D_{KL}(q_\encb(\zb|\xb)||p_{\decb}(\zb))\nonumber\\ &<& \log p_{\decb}(\xb).
  \end{eqnarray}

After training the objective, by sampling from the $p_{\decb}(\zb)=\N(\zb|\textbf{0},I_H)$ or setting $\zb$ with purpose, the learned $p_{\decb}(\xb|\zb)$ can generate new samples.

    \subsection{Classification}
The  $q_\encb(\zb|\xb)$ can further support latter tasks such as classification. Let $p_{\preb}(\yb|\zb)$ denote the predicting process, and the classification objective is the following,
  \begin{equation}\label{Approximate Inference for classification}
\L_{pre}=\E\limits_{\zb\sim q_\encb(\zb|\xb)}\log p_{\preb}(\yb|\zb).
  \end{equation}
% when joint optimizing the two objective and let  $p_{\decb,\preb}(\xb,\yb)=\int p_{\decb}(\xb|\zb)p_{\preb}(\yb|\zb) p_{\decb}(\zb)d\zb$,
%  \begin{equation}\label{Approximate Inference for joint optimization}
%\L_{pre}+\L_{rec}\le\log p_{\decb,\preb}(\xb,\yb),
%  \end{equation}
%  with equality holds iff
%    \begin{equation}\label{Equality Condition}
%D_{KL}(q_\encb(\zb|\xb)||p_{\decb,\preb}(\zb|\xb,\yb))=0.
%  \end{equation}
Let $\Yb_{\preb}$ denote the random variables with $p_\preb(\yb)=\int_{}^{}p_\preb(\yb|\zb)q_\encb(\zb|\xb) p_{data}(\xb)d\xb$.

In real implementation the above objectives should further take expectation on the data distribution. However, sometimes only part factors are manually found useful for the generation (\cite{goodfellow2016deep}), and the factor which is irrelevant to $\xb$ can not support classification either. Therefore, some approaches to automatically find the influential factor beneficial to the latter tasks are demanded.

\section{Mutual Information as A Necessary Indicator}
By exploring why factors are ignored, we argue that mutual information is a necessary indicator to find the influential factor.
\subsection{Ignored Factor Analysis}

\subsubsection{Low Intrinsic Dimension of Data}
One aim of the VAE is to learn the data intrinsic factors but intrinsic dimension keeps the same under the continuous reversible mapping suggested by Theorem~\ref{Information Conservation}.

\begin{theorem}[Information Conservation]\label{Information Conservation} Suppose that there are two sets of $H$ and $P$ ($H\neq P$), $\Zb=(Z_1,\cdots,Z_H)$ and $\Yb=(Y_1,\cdots,Y_P)$, respectively, independent unit Gaussian random variables, then these two sets of random variables can not be the generating factor of each other. That is, there are no continuous functions $f:\R^H \rightarrow \R^P$ and $g:\R^P \rightarrow \R^H$ such that
$$\Zb=g(\Yb)\quad and\quad \Yb=f(\Zb).$$
\end{theorem}
\begin{proof}\label{pf:Information Conservation}
   Proof by Contradiction. Suppose those two function exist, and we will show that they will be  inverse mapping of each other and  $f$ is a homeomorphism mapping of $\R^H$ and $\R^P$. Here, $f$ is said to be a homeomorphism mapping if it satisfies the following three conditions:
  \begin{itemize}
    \item  $f$ is a bijection,
    \item  $f$ is continuous,
    \item   the inverse function $f^{-1}$ is continuous.
  \end{itemize}
  Since $\R^H$ and $\R^P$ have different topology structures ($P\neq H$), the homeomorphism mapping will not exist.

\begin{center}
  $\Zb=g(\Yb)=g(f(\Zb))$ $\forall \Zb\in \R^H$ $\Rightarrow$ $g\circ f=I_H$
\end{center}
\begin{center}
  $\Yb=f(\Zb)=f(g(\Yb))$ $\forall \Yb\in \R^P$ $\Rightarrow$ $f\circ g=I_P$
\end{center}
It yields $g$ is the inverse function of $f$ and $f$ is bijection.
Since both $f$ and $g$ are continuous, $f$ is a homeomorphism mapping between $\R^H$ and $\R^P$ and it leads to the contradiction.
\end{proof}
%Proof is listed in Appendix~\ref{pf:Information Conservation}.
 Suppose the oracle data, denoted by random variable $\Xb$, is generated by  $\Yb$ (with $P$ independent unit Gaussian random variables)
with a homeomorphism mapping $\Xb=\phi(\Yb)$. Factors $\Zb$ (with $H$ independent unit Gaussian random variables) generates
the $\Xb$ with a homeomorphism mapping $\Xb=\psi(\Zb)$. It yields $\Zb=\psi^{-1} \circ \phi (\Yb)$ and $\Yb=\phi^{-1}\circ \psi(\Zb)$.
Then according to the information conservation theorem, it must hold that $H=P$.

%\begin{center}
%  \begin{figure*}[t]
%  \centering
%  \includegraphics[width=0.7\linewidth, clip=true, trim= 0 50 0 60 ]{generatortraversal/NoiseModelling/InformationConservation.eps}
%  \caption{\textbf{Intrinsic Dimension Keeps the Same Under the Continuous Reversible Mapping.} The illustration of the information conservation theorem~\ref{Information Conservation}. Suppose that the oracle data, denoted by random variable $\xb$, is generated by  $\yb$ (with $P$ independent unit Gaussian random variables)
%with a homeomorphism mapping $\xb=\phi(\yb)$. Factors $\zb$ (with $H$ independent unit Gaussian random variables) generates
%the $\xb$ with a homeomorphism mapping $\xb=\psi(\zb)$. It yields $\zb=\psi^{-1} \circ \phi (\yb)$ and $\yb=\phi^{-1}\circ \psi(\zb)$.
%Then according to the information conservation theorem, it must hold that $H=P$.}
%   \label{fig:Information Conservation}
%\end{figure*}
%\end{center}
%杩欎釜瀹氱悊濡傛灉涓嶅鏄犲皠鐨勫舰寮忚繘琛岄檺鍒剁殑璇濓紝閭ｄ箞鐩稿簲鐨勫嚱鏁颁篃璁哥湡鐨勫瓨鍦?
%%鑳藉惁璇存槑f鏄竴鑷磋繛缁殑鏃跺€欏氨宸茬粡涓嶈兘瀹炵幇浠庝綆鍒伴珮鐨勬槧灏勪簡鍛紵

For example, 10 Gaussian factors and 128 Gaussian factors can not generate each other.  Analogically, if the data are generated by 10 intrinsic Gaussian factors, it intuitively would not be inferred to 128 Gaussian factors by the VAE and some factors would be independent with data although we may pre-set in this way.

\subsubsection{Mutual Information Reflexes the Absolute Statistic Dependence}
In order to quantify the dependence and estimate which factor influences the generating process or has no effect at all,
the mutual information of $\zb_{\encb}$ and $\xb$, $I(\Xb;{Z_{\encb}}_h)$ can be taken as a rational indicator(\cite{peng2005feature})
. That is,
\begin{equation}\label{mutual information of encoder}
  I(\Xb;{Z_{\encb}}_h)=\E_{\xb\sim p_{data}(\xb)}D_{KL}(q_\encb(z_h|\xb)||q_\encb(z_h)).
\end{equation}

The mutual information can reflect the absolute statistic dependence: $I(\Xb;{Z_{\encb}}_h)=0$ if and only if $\Xb$ and ${Z_{\encb}}_h$ are independent. The larger $I(\Xb;{Z_{\encb}}_h)$ is, the more information ${Z_{\encb}}_h$ conveys regarding $\Xb$, and the more influential factor it should be to represent the data.
%The less $I(\Xb;{Z_{\encb}}_h)$, the more similar $q_{\encb}(\xb,\zb_{\encb})$ to $p_{data}{\encb}(\xb)q(\zb_{\encb})$.

\subsubsection{Sparsity in Mutual Information\label{Information Channel}}
Actually, mutual information is implicitly involved in the VAE objective. The following theorem further suggests the VAE objective induces the sparsity in mutual information. It then explains why factors are ignored from the perspective of mutual information.

\begin{theorem}[Objective Decomposition]\label{Decomposition of Objective} If $q_\encb(\zb|\xb) \ll p_\decb(\zb)$\footnote{That is the support of $q_\encb(\zb|\xb)$ is contained in  support of  $p_\decb(\zb)$.}, for any $\xb$, $q_\encb(\zb|\xb)=q_\encb(z_1|\xb)\cdots q_\encb(z_H|\xb)$ and $p_{\decb}(\zb)=\N(\zb|\textbf{0},I_H)$ then it yields the following decomposition:
\begin{itemize}
  \item  $L_1$ norm expression of the KL-divergence term in the VAE:

  \small
  \begin{eqnarray}
  &&\E_{\xb\sim p_{data}(\xb)}D_{KL}(q_\encb(\zb|\xb)||p_\decb(\zb))\nonumber\\ &=&\sum_{h=1}^{H}\E_{\xb\sim p_{data}(\xb)}D_{KL}(q_\encb(z_h|\xb)||p_\decb(z_h))\nonumber\\&=&\Vert (\E_{\xb\sim p_{data}(\xb)}D_{KL}(q_\encb(z_1|\xb)||p_\decb(z_1)),\nonumber\\&&\E_{\xb\sim p_{data}(\xb)}D_{KL}(q_\encb(z_2|\xb)||p_\decb(z_2)),\cdots,\nonumber\\&&\E_{\xb\sim p_{data}(\xb)}D_{KL}(q_\encb(z_h|\xb)||p_\decb(z_h)))\Vert_1.
  \end{eqnarray}
  \normalsize

  \item Further decomposition of an entity in the $L_1$ norm expression\footnote{It is similar to the result in \cite{hoffman2016elbo}.}:
  \begin{eqnarray}
  &&\E_{\xb\sim p_{data}(\xb)}D_{KL}(q_\encb(z_h|\xb)||p_\decb(z_h))\nonumber\\&=&I(\Xb;{Z_{\encb}}_h)+D_{KL}(q_\encb(z_h)||p_\decb(z_h)).
  \end{eqnarray}
\end{itemize}
\end{theorem}
\begin{proof} \label{pf:Decomposition of Objectiven}
The $L_1$ norm expression is obvious. We  prove the further decomposition of an entity in the $L_1$ norm expression:
\begin{equation}
   \E_{\xb\sim p_{data}(\xb)}D_{KL}(q_\encb(z_h|\xb)||p_\decb(z_h))=\int q_\encb(z_h|\xb)p_{data}(\xb)\frac{q_\encb(z_h|\xb)p_{data}(\xb)}{p_\decb(z_h)p_{data}(\xb)}d\xb $$$$= \int q_\encb(z_h|\xb)p_{data}(\xb)\frac{q_\decb(z_h|\xb)p_{data}(\xb)}{q_\encb(z_h)p_{data}(\xb)} \frac{q_\encb(z_h)}{p_\decb(z_h)}d\xb$$$$=I(\Xb; {Z_{\encb}}_h)+D_{KL}(q_\encb(z_h)||p_\decb(z_h)).
\end{equation}
\end{proof}
%Proof is given in Appendix~\ref{pf:Decomposition of Objectiven}.
The theorem demonstrates that the expectation of the second term in variation lower bound in Eq.~(\ref{Approximate Inference}) can be represented in the form of $L_1$ norm which inclines to induce the sparsity of $I(\Xb;{Z_{\encb}}_h)$ and $D_{KL}(q_\encb(z_h)||p_\decb(z_h))$ together in $h$, clipping down the non-intrinsic factor dimension to some extent. The sparsity of Expectation $\E_{\xb\sim p_{data}(\xb)}D_{KL}(q_\encb(z_h|\xb)||p_\decb(z_h))$ actually leads to sparsity of both its summarization terms $I(\Xb;{Z_{\encb}}_h)$ and $D_{KL}(q_\encb(z_h)||p_\decb(z_h))$ together in $h$, since both of them are non-negative. For any zero value summarization, both of its elements should also be zero.
Thus this regularization term inclines to intrinsically conduct sparsity of mutual information $I(\Xb;{Z_{\encb}}_h)$, which have been comprehensively substantiated by all our experiments, as can be easily seen in Fig.(\ref{fig:Iencoderzhsigma2_argumented}) and Fig.(\ref{fig:mutual information sparsity in DEAP and celeba}).

\begin{figure*}
\centering

 \includegraphics[width=\textwidth]{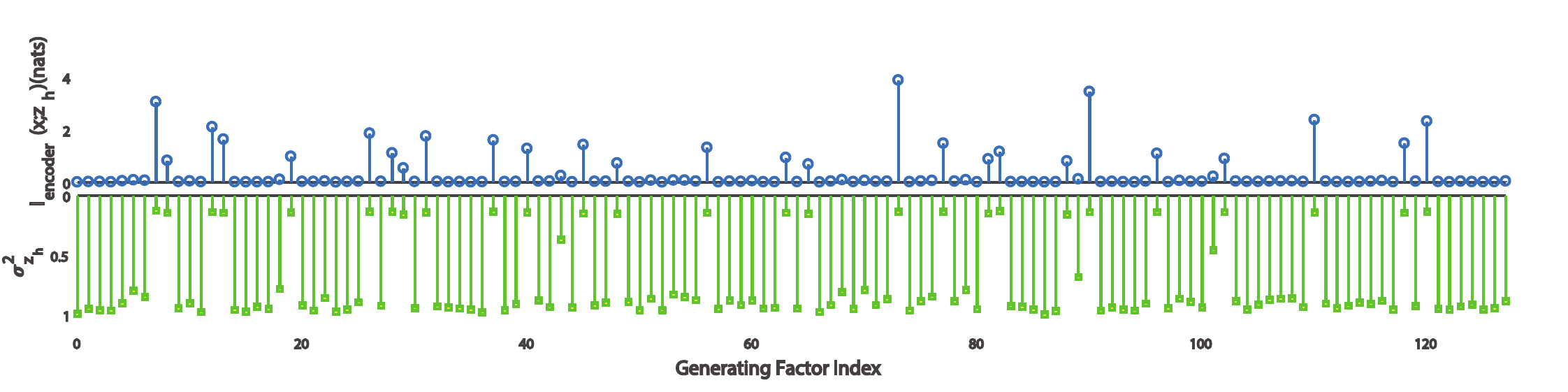}
{(A) $I(\Xb;{Z_{\encb}}_h)$, $\sigma^2_{z_h}$ plot of $\beta(=40)$-VAE on CelebA.}

 \includegraphics[width=\textwidth]{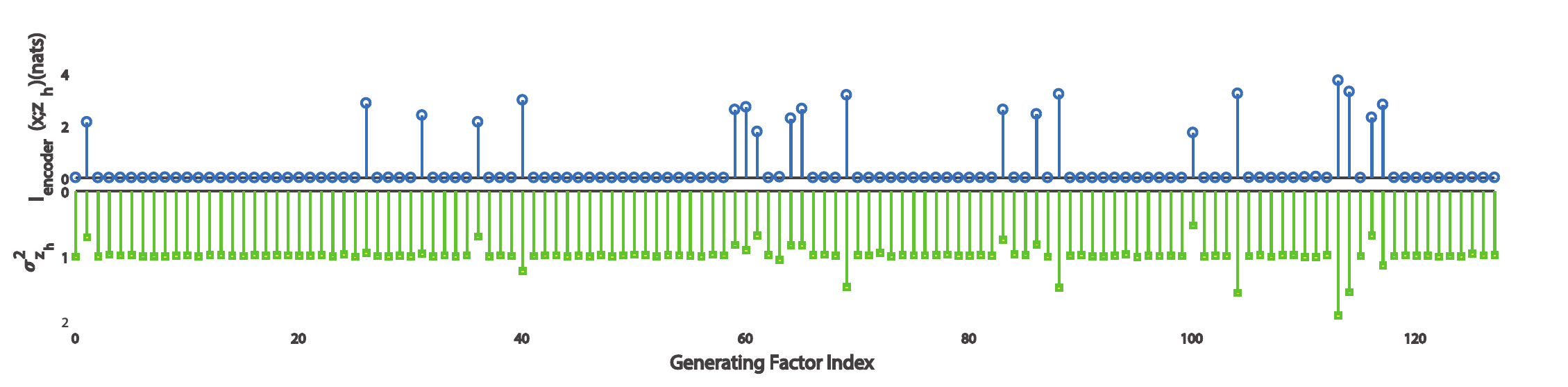}
 {(B) $I(\Xb;{Z_{\encb}}_h)$, $\sigma^2_{z_h}$ plot of $\beta(=6)$-VAE on DEAP.}
\caption{\textbf{Mutual information sparsity occurs on CelebA and DEAP.} \label{fig:mutual information sparsity in DEAP and celeba}}
\end{figure*}

Therefore, the VAE objective inclines to induce mutual information sparsity in factor dimension over the data intrinsic dimension and the factor ignored phenomenon occurs. On the one hand, with increase in the KL divergence regularization, even when the number of latent factors is set large, unlike auto-encoder, the over-fitting issue still tends not to occur. On the other hand, this  helps us get influential factors to represent the variants of data, and facilitate an efficient generalization of data by varying these useful factors while neglecting others.

By the way, the following theorem suggests the condition that we can use $I(\xb;z_h)$ to estimate the whole mutual information.

\begin{theorem}[Mutual Information Separation]\label{Separation of the Mutual Information} Let $Z_1,\cdots,Z_H$ be independent unit Gaussian distribution, and $Z_1,Z_2,\cdots,Z_H$ be conditional independent given $\Xb$. Then
  \begin{eqnarray}\label{mutualinfoseparation}I(\Xb;Z_1,\cdots,Z_H)&=&\sum_{h=1}^{H}I(\Xb;Z_h)\nonumber\\&=&\Vert (I(\Xb;Z_1),I(\Xb;Z_2),\cdots,I(\Xb;Z_H))\Vert_1.\nonumber\\
  \end{eqnarray}
\end{theorem}
\begin{proof} \label{pf:Separation of the Mutual Information}
  $$I(\Xb;Z_1,\cdots,Z_H)=\int p(z_1,\cdots,z_H,\xb)\log \frac{p(\xb,z_1,\cdots,z_H)}{p(z_1,\cdots,z_H)p(\xb)}dz_1\cdots dz_Hdx$$$$=\int p(\xb,z_1,\cdots,z_H)\log \frac{\prod_{h=1}^H p(z_h|\xb)}{\prod_{h=1}^H p(z_h)}dz_1\cdots dz_Hd\xb$$$$=\sum_{h=1}^{H}\int p(\xb,z_h)\log \frac{ p(z_h|\xb)}{ p(z_h)}dz_hd\xb=\sum_{h=1}^{H}I(\Xb;Z_h).$$
\end{proof}
%Proof is presented in Appendix~\ref{pf:Separation of the Mutual Information}.
This theorem suggests that if the learnt $q_\encb(\zb)$ can factorize and the $q_\encb(\zb|\xb)$ can factorize,
then we could use the sum of $I(\xb; {z_{\encb}}_h)$ to direct estimate the whole mutual information.

\subsection{Reconstruction and Classification  Theoretical Supports}

According to \cite{cover2012elements}, the mutual information can also provide a lower bound for the best mean recover error.
\begin{theorem} \label{thm: MSE and MI}Suppose $\Xb$ is with differential entropy $H(\Xb)$, then let $\hat{\Xb}(\Zb_{\encb})$ be an estimation of $\Xb$, and give side information $\Zb_{\encb}$\footnote{ Notice that $\Zb_{\encb}$ are random variables with $q_\encb(\zb)=\int_{}^{}q_\encb(\zb|\xb)p_{data}(\xb)d\xb$. $\hat \Xb(\Zb_{\encb})$ is a function named $\hat \Xb$ of $\Zb_{\encb}$.}, and then it holds that
\begin{equation}
  \E(\Xb-\hat{\Xb}(\Zb_{\encb}))^2\ge \frac{1}{2 \pi e} e^{2 (H(\Xb)-I(\Xb;\Zb_{\encb}))}.
\end{equation}
\end{theorem}

Therefore, if we set $p_\decb(\xb|\zb)=\N(\xb|\decb_\mu(\zb),\decb_\sigma)$, then $\Xb_{rec}=\decb_\mu(\Zb_{\encb})$ has $ \frac{1}{2 \pi e} e^{2 (H(\Xb)-I(\Xb;\Zb_{\encb}))}$ as the lower bound for recovering.  Let $\Zb_{\encb}=[\Zb_{major}, \Zb_{minor}]$. Let us only use a major set  of  factors $\Zb_{major}$, that is to construct a new estimator $\Xb_{recc}=\decb_\mu(\Zb_{major},\mathbf{0})$ with setting $Z_{minor}=\mathbf{0}$. With the assumption that $q_\encb(\zb)$ can factorize, it yields the separation of mutual information $I(\Xb;\Zb_{\encb})=I(\Xb;\Zb_{minor})+I(\Xb;\Zb_{major})$. It yields the following bound,
\begin{eqnarray}\label{}
  &&\E(\Xb-\Xb_{recc})^2\nonumber \\&\ge& \frac{1}{2 \pi e} e^{2 (H(\Xb)-I(\Xb;\Zb_{major}))}\nonumber\\ &\ge& \frac{1}{2 \pi e} e^{2 (H(\Xb)-I(\Xb;\Zb_{major}))}e^{-2I(\Xb;\Zb_{minor})}.
\end{eqnarray}

The theorem implies that the mutual information carried by the selecting factors directly influences on the lower bound of the best reconstruction and we may select some top influential factors carrying the most information to represent and generate the data with less reconstruction distortion.

We further provide some theoretical supports for the proposed mutual information as the factor indicator in classification.

Suppose that Markov chain condition,  $\Yb\to\Xb\to\Zb_{\encb} \to \Yb_{pre}$, holds.\footnote{ This condition implies $\Yb\to \Xb\to\Zb_{\encb}$ which guarantees that $q_\encb(\zb|\xb,\yb)=q_\encb(\zb|\xb)$. It also implies $\Yb\to\Zb_{\encb}\to\Yb_{\preb}$ which guarantees that $\Yb_{\preb}$ can be then taken as a rational estimator of $\Yb$ based on Theorem 5.} According to the Fano's inequality (\cite{cover2012elements}) and the information processing inequality the mutual information also correlates with the classification error.

\begin{theorem}[Fano's inequality] For any estimation $\hat\Yb$ such that $\Yb\to\Zb_{\encb} \to \hat\Yb$, with $P_e=Pr(\hat\Yb \neq \Yb )$, we have
\begin{equation}\label{}
 H(P_e)+P_e log \vert \mathcal{Y}\vert \ge H(\Yb)-I(\Yb;\Zb_{\encb})\ge H(\Yb)-I(\Xb;{\Zb_{\encb}})
\end{equation}
where $\mathcal{Y}$ is the alphabet of $\Yb$.
Since the number of class is no smaller than 2, it naturally holds that $log(|\mathcal{Y}|) > 0$. This inequality can then pbe weakened to
\begin{equation}\label{}
 1+P_e log \vert \mathcal{Y}\vert \ge H(\Yb)-I(\Yb;\Zb_{\encb})\ge H(\Yb)-I(\Xb;{\Zb_{\encb}}).
\end{equation}
or
\begin{equation}\label{}
 P_e\ge \frac{H(\Yb)-I(\Yb;\Zb_{\encb})-1}{log \vert \mathcal{Y} \vert }\ge \frac{H(\Yb)-I(\Xb;{\Zb_{\encb}})-1}{log \vert \mathcal{Y} \vert}.
\end{equation}
\end{theorem}

Note that according to information processing inequality  $I(\Xb;{Z_{\encb}}_h)\ge I(\yb;{z_{\encb}}_h)$.  $I(\Xb;{Z_{\encb}}_h)=0\Rightarrow I(\yb;{z_{\encb}}_h)=0$, and If $I(\yb;{z_{\encb}}_h)=0$ the $h^{th}$ factor will not influence the prediction. Let we regard $\Yb_{pre}$ as $\hat\Yb$. With the assumption that $q_\encb(\zb)$ can factorize, since $\Yb\to\Zb_{\encb} \to \Yb_{pre}$, the theorem suggests the mutual information carried by the selecting factors directly influences the lower bound of the classification error and therefore we can remove minor factors according to the mutual information $I(\Xb;{Z_{\encb}}_h)$ without significantly lifting the lower bound of the prediction error.

\subsection{Algorithms to Quantitatively Calculate the  Proposed indicators}
In order to calculate $I(\Xb;\Zb_{\encb})$, we assume that $q^*(\zb)=\N(\zb|\textbf{0},diag(\sigma^*_1,\cdots,\sigma^*_H))$ is a factorized zero mean Gaussian estimation for $q_\encb(\zb)$.

We can then list the indicators to be estimated as:

%\begin{definition}[Estimation for $\E\limits_{\xb\sim p_{data}(\xb)}D_{KL}(q_\encb(\zb|\xb)||p_\decb(\zb))$]\label{Empirical Estimation KL(qzx||pz)}
%  \begin{equation}
%    \tilde{D}_{KL}(q_\encb(\zb|\xb)||p_\decb(\zb))=\frac{1}{M}\sum_{m=1}^{M}D_{KL}(q_\encb(\zb|\xb_m)||p_\decb(\zb)).
%  \end{equation}
%\end{definition}

\begin{definition}[Estimation for $I(\Xb;\Zb_{\encb})$: the information conveyed by whole factors]\label{Empirical Estimation For Iencoder}
  \begin{equation}
    I_{est}(\Xb;\Zb_{\encb})_M=\frac{1}{M}\sum_{m=1}^{M}D_{KL}(q_\encb(\zb|\xb^m)||q^*(\zb)).
  \end{equation}
\end{definition}
This estimation uses $M$ sample according to the empirical form of Corollary \ref{close estimation for mutual information}.

\begin{definition}[Estimation for $I(\Xb;{Z_{\encb}}_h)$: the information conveyed by a factor]\label{Empirical Estimation For Iencoder(xzh)}
  \begin{equation}
    I_{est}(\Xb;{Z_{\encb}}_h)_M=\frac{1}{M}\sum_{m=1}^{M}D_{KL}(q_\encb({z_{\encb}}_h|\xb^m)||q^*({z_{\encb}}_h)).
  \end{equation}
\end{definition}
This indicator quantifies mutual information of a specific factor and input data.
%
%\begin{definition}[Estimation for $D_{KL}(q_\encb(\zb)||p_\decb(\zb))$]\label{Empirical Estimation For KL(qz||pz)}
%  \begin{equation}
%     \tilde{D}_{KL}(q_\encb(\zb)||p_\decb(\zb))= $$$$  \tilde{D}_{KL}(q_\encb(\zb|\xb)||p_\decb(\zb))-\tilde{I}_{encoder}(\xb;\zb).
%  \end{equation}
%\end{definition}

Note that the above indicators  need the value of $q^*(\zb)$, and thus we need to design algorithms to calculate this term. Based on Theorem~\ref{Decomposition of Objective} through the minimization equivalence, we know that
\begin{equation}\label{Minimization Equivalence}
  \min_{q}\E_{\xb\sim p_{data}(\xb)}D_{KL}(q_\encb(\zb|\xb)||q(\zb))$$$$\Leftrightarrow \min_{q}D_{KL}(q_\encb(\zb)||q(\zb)) d\zb,
\end{equation}

and then we can prove the following result:

\begin{corollary}\label{close estimation for mutual information} if $q_\encb(\zb|\xb) \ll  q^*(\zb)$ then
    \begin{eqnarray}
       &&\E_{\xb\sim p_{data}(\xb)}D_{KL}(q_\encb(\zb|\xb)||q^*(\zb))\nonumber\\&=&I(\Xb;\Zb_{\encb})+D_{KL}(q_\encb(\zb)||q^{*}(\zb)).
    \end{eqnarray}
\end{corollary}
The proof of  Corollary \ref{close estimation for mutual information} is the same as that of Theorem~\ref{Decomposition of Objective}. This corollary suggests that the estimation defined in Definition~\ref{Empirical Estimation For Iencoder} provides another upper bound for the capacity of the encoder network. Empirically, this estimation is a much tighter estimation than the second term of the Objective~(\ref{Approximate Inference}).

$q^*(\zb)$ can then be obtained  by solving the following optimization problem:
\begin{equation}\label{Obtain q(z)}
  q^*(\zb)=\arg\min_{q}\frac{1}{M}\sum_{m=1}^{M}D_{KL}(q_\encb(\zb|\xb^m)||q(\zb)).
\end{equation}
The above minimization problem can be solved with a closed-form solution as follows\footnote{The above minimization problem can also be solved by gradient descent.}:
$$\sigma^*_i=\frac{\sum_{m=1}^{M}\sigma_i(\xb^m)+\mu_i^2(\xb^m)}{M}.$$
The proof is as follows:
%\textbf{(!!!please add proof here!!!})

Notice that suppose we have  two multivariate normal distributions, with means $\mu _{0},\mu _{1}$ and with non-singular covariance matrices $\Sigma _{0},\Sigma _{1}$ and the two distributions have the same dimension H, then it yields\cite{duchi2007derivations} \begin{equation}\label{KL-divergence of Gaussians}
  D_{KL}(\N_0||\N_1)=\frac{1}{2}(tr(\Sigma_1^{-1}\Sigma_0)+(\mu_1-\mu_0)^T\Sigma_1^{-1}(\mu_1-\mu_0)-H+\ln(\frac{\det\Sigma_1}{\det\Sigma_0})).
\end{equation}
Note that we assume $q(\zb)=\N(\zb|\textbf{0},diag(\sigma_1,\cdots,\sigma_H))$ and

 $q_\encb(\zb|\xb)=\N(\zb |${\boldmath$\mu$}$(\xb),diag(\sigma_1(\xb),\cdots,\sigma_H(\xb)))$. Thus, we have
\begin{eqnarray}
% \nonumber % Remove numbering (before each equation)
  &&\frac{1}{M}\sum_{m=1}^{M}D_{KL}(q_\encb(\zb|\xb^m)||q(\zb))\nonumber\\  &=&\frac{1}{M}\sum_{m=1}^{M}\frac{1}{2}\sum_{i=1}^{H}\frac{\sigma_i(\xb^m)}{\sigma_i}+\frac{\mu_i(\xb^m)^2}{\sigma_i} -1 +\ln\frac{\sigma_i}{\sigma_i(\xb^m)} \nonumber\\
  &=&  \frac{1}{2M}\sum_{i=1}^{H}\sum_{m=1}^{M}\frac{\sigma_i(\xb^m)}{\sigma_i}+\frac{\mu_i(\xb^m)^2}{\sigma_i} -1 +\ln\frac{\sigma_i}{\sigma_i(\xb^m)}.\nonumber
\end{eqnarray}
The optimization can be divided into H optimization sub-problems as the following,
\begin{equation}\label{subproblem}
  \sigma_i^*=\arg\min_{\sigma_i}\sum_{m=1}^{M}\frac{\sigma_i(\xb^m)}{\sigma_i}+\frac{\mu_i(\xb^m)^2}{\sigma_i} -1 +\ln\frac{\sigma_i}{\sigma_i(\xb^m)}, i=1,\cdots,H.
\end{equation}
\begin{equation}\label{subproblem}
  \nabla_{\sigma_i}(\sum_{m=1}^{M}\frac{\sigma_i(\xb^m)}{\sigma_i}+\frac{\mu_i(\xb^m)^2}{\sigma_i} -1 +\ln\frac{\sigma_i}{\sigma_i(\xb^m)})=\sum_{m=1}^{M}-\frac{\sigma_i(\xb^m)+\mu_i(\xb^m)^2}{\sigma_i^2}+\frac{1}{\sigma_i}.
\end{equation}
Since
\begin{equation}\label{subproblem solving}
\sum_{m=1}^{M}-\frac{\sigma_i(\xb^m)+\mu_i(\xb^m)^2}{{\sigma_i^*}^2}+\frac{1}{\sigma_i^*}=0,
\end{equation}
it yields
\begin{equation}\label{subproblem solving result}
\sigma_i^*=\frac{\sum_{m=1}^{M}\sigma_i(\xb^m)+\mu_i(\xb^m)^2}{M}.
\end{equation}

The above procedure is summarized and presented in  Algorithm~\ref{alg:example} to calculate the proposed indicators.

\begin{algorithm}[]
   \caption{Mutual Information Estimation}
   \label{alg:example}
\begin{algorithmic}[1]
   \STATE {\bfseries Input:} Sampled Data $\{\xb^m\}_{m=1}^{M}$,\\ Encoder Network $q_\encb(\zb|\xb)=\N(\zb |${\boldmath$\mathbf{\mu}$}$(\xb),diag(\sigma_1(\xb),\cdots,\sigma_H(\xb)))$
    \STATE {\bfseries Obtain:} $q^*(\zb)=\arg\min_{q}\frac{1}{M}\sum_{m=1}^{M}D_{KL}(q_\encb(\zb|\xb^m)||q(\zb)).$
\FOR{$i=h$ {\bfseries to} $H$}
\STATE $\sigma^*_i=\frac{\sum_{m=1}^{M}\sigma_i(\xb^m)+\mu_i^2(\xb^m)}{M}.$  \ENDFOR
\STATE {\bfseries   Calculate:}    $I_{est}(\Xb;{Z_{\encb}}_h)_M=\frac{1}{M}\sum_{m=1}^{M}D_{KL}(q_\encb({z_{\encb}}_h|\xb^m)||q^*({z_{\encb}}_h))$
\FOR{$i=h$ {\bfseries to} $H$}
\STATE $I_{est}(\Xb;{Z_{\encb}}_h)_M=\frac{1}{M}\sum_{m=1}^{M}\frac{1}{2}(\log\frac{\sigma^*_h}{\sigma_h(\xb^m)})$.
 \ENDFOR
\STATE {\bfseries Calculate:}
$I_{est}(\Xb;{Z_{\encb}})_M$
 \STATE
$I_{est}(\Xb;\Zb_{\encb})_M=\sum_{h=1}^H I_{est}(\Xb;{Z_{\encb}}_h)_M.$
  \STATE {\bfseries Output: $I_{est}(\Xb;\Zb_{\encb})_M$, $I_{est}(\Xb;{Z_{\encb}}_h)_M$, $q^*(\zb)$}
\end{algorithmic}
\end{algorithm}

%\begin{corollary}The terminology follows the aforemention definitions and if the involved KL-divergence and mutual information be well defined then
%    \begin{equation}
%      \E_{\xb\sim p_{data}(\xb)}D_{KL}(q_\encb(\zb|\xb)||p_\theta(\zb))- \E_{\xb\sim p_{data}(\xb)}D_{KL}(q_\encb(\zb|\xb)||q^*(\zb))$$$$=D_{KL}(q_\encb(\zb)||p_\theta(\zb))-D_{KL}(q_\encb(\zb)||q^*(\zb))\le D_{KL}(q_\encb(\zb)||p_\theta(\zb)).
%    \end{equation}
%\end{corollary}
%The corollary is an direct result of theorem~\ref{Decomposition of Objective} and corollary~\ref{close estimation for mutual information}. It suggests that the estimation in definition~\ref{Empirical Estimation For KL(qz||pz)} is a lower bound for $D_{KL}(q_\encb(\zb)||p_\theta(\zb))$.

The following definition and theorem clarify the consistency of the estimation on mutual information.
\begin{definition}[Consistency] The estimator $I_{est}(\Xb;\Zb_{\encb})_M$ is consistent to $I(\Xb;\Zb_{\encb})$ if and only if: $\forall \varepsilon>0$ $\forall \delta>0$, $\exists N$, and $q^*(\zb)$, $\forall M>N$, with probability greater than $1-\delta$, we have
\begin{equation}\label{}
  \vert I_{est}(\Xb;\Zb_{\encb})_M -  I(\Xb;\Zb_{\encb}) \vert<\varepsilon.
\end{equation}

\end{definition}

\begin{theorem}\label{Consistency}  The estimator $I_{est}(\Xb;\Zb_{\encb})_M$ is consistent to  $I(\Xb;\Zb_{\encb})$. That is, if the choice of $q^{*}(\zb)$ satisfied the condition that $D_{KL}(q_\encb(\zb)||q^{*}(\zb))<\varepsilon/2$, then $\forall \delta>0$, $\exists N$, $\forall M>N$, with probability greater than $1-\delta$, we have
\begin{equation}\label{}
  \vert I_{est}(\Xb;\Zb_{\encb})_M -  I(\Xb;\Zb_{\encb}) \vert<\varepsilon.
\end{equation}

\end{theorem}
\begin{proof}
  Let $\tilde{I}[q^*] = \E_{\xb\sim p_{data}(\xb)}D_{KL}(q_\encb(\zb|\xb)||q^*(\zb))$. According to the law of big number, we have $\forall \delta>0$, $\exists N$, $\forall M>N$, with probability greater than $1-\delta$, we have
\begin{equation}\label{}
  \vert I_{est}(\Xb;\Zb_{\encb})_M -  \tilde{I}[q^*] \vert<\varepsilon/2,
\end{equation}
\begin{eqnarray}
% \nonumber % Remove numbering (before each equation)
   && \vert I_{est}(\Xb;\Zb_{\encb})_M -  I(\Xb;\Zb_{\encb}) \vert\nonumber\\&\le&  \vert I_{est}(\Xb;\Zb_{\encb})_M -  \tilde{I}[q^*]  \vert + \vert \tilde{I}[q^*] - I(\Xb;\Zb_{\encb}) \vert  \nonumber\\ &<&  \frac{\varepsilon}{2}+\vert D_{KL}(q_\encb(\zb)||q^{*}(\zb))\vert<   \varepsilon.
\end{eqnarray}
\end{proof}
This theorem suggests the estimation under a high probability could be arbitrary close to the real mutual information  provided that the estimation $q^*(\zb)$ is arbitrary close to the learned $q_\encb(\zb)$ and the number of the sample is bigger enough. Besides, the minimization of $D_{KL}(q_\encb(\zb)||q^{*}(\zb))$ in theorem 6 inspires the derivation of $q^{*}(\zb)$.

\section{Related Work}

There are not too many works on such indicator designing issue to discover influential factors in the VAE. A general and easy approach for determine the
 VAE's factor influence  is through intuitive visual (\cite{goodfellow2016deep}, \cite{higgins2016beta}) or  aural (\cite{hsu2017unsupervised}) observation. However, it might be labor-intensive to select factors for latter tasks.

In (\cite{alemi2016deep}), $q_\encb(\zb|\xb)$ are visualized by plotting the 95\% confidence interval as an ellipse to supervise the behavior of network and  it reflects the factor influence directly. However, it still needs human to interpret the plot.

In classical PCA, it's common to select factor with high variance and (\cite{higgins2017scan}) suggests that the variance of factor may indicate the usage of the factors. However, the variance could not always represent the absolute statistical relationship between the factors and data, which can be easily observed by Fig.(\ref{fig:Iencoderzhsigma2_argumented}) and Fig.(\ref{fig:mutual information sparsity in DEAP and celeba}).

Our work emphasizes mutual information which conveys the absolute statistical relationship between the factors and the data and uses it as an indictor to find the influential factors, substantiated with the relationship of the total information of selected factors and the reconstruction and relationship of mutual information and classification. All our experiments substantiate that designed indicator can discover the influential factors significantly relevant for data representation.

\section{Experimental Results}

\subsection{Datasets}

MNIST is a database of handwritten digits  (\cite{L1998Gradient}). We estimate all mutual information of factors learned  from it and then use different ratio of top influential factors for the latter generation task.

CelebA (\cite{liu2015deep}) is a large-scale celebfaces attributes datasets and we only use its images to sustain influential factor discovery.

DEAP is a publicly famous multi-modalities emotion recognition dataset proposed by (\cite{koelstra2012DEAP}) and we use the transformed the signal -to-video sequence for 4-class emotion prediction by using different ratio of top influential factors and for emotion relevant influential factor extraction.

More details are presented in ~\ref{Experiment Details}.
\subsection{Influential Factor Discovery Tests}

According to Fig.(\ref{fig:Iencoderzhsigma2_argumented}), the proposed mutual information estimator effectively determines the influential as well as the non-influential factors. The factors with small values of estimated mutual information can be found with little generation effects and  factors with large values of mutual information can be found with influential generation effects. Comparatively, it can be observed that the variance as used in classical methods can not significantly indicate the usage of factors.

\begin{figure*}
\centering
\hspace*{-0.003\textwidth}\begin{minipage}{0.32\textwidth}
             \centering
             %\vspace*{-0.1\textwidth}
                          {\tiny Factor 7: $I(\Xb;{Z_{\encb}}_7)=3.1$,\\ $\mathop{\longrightarrow}\limits^{Background\ Brighten}$ }
        \includegraphics[width=\textwidth, clip=true, trim = 0 0 0 0]{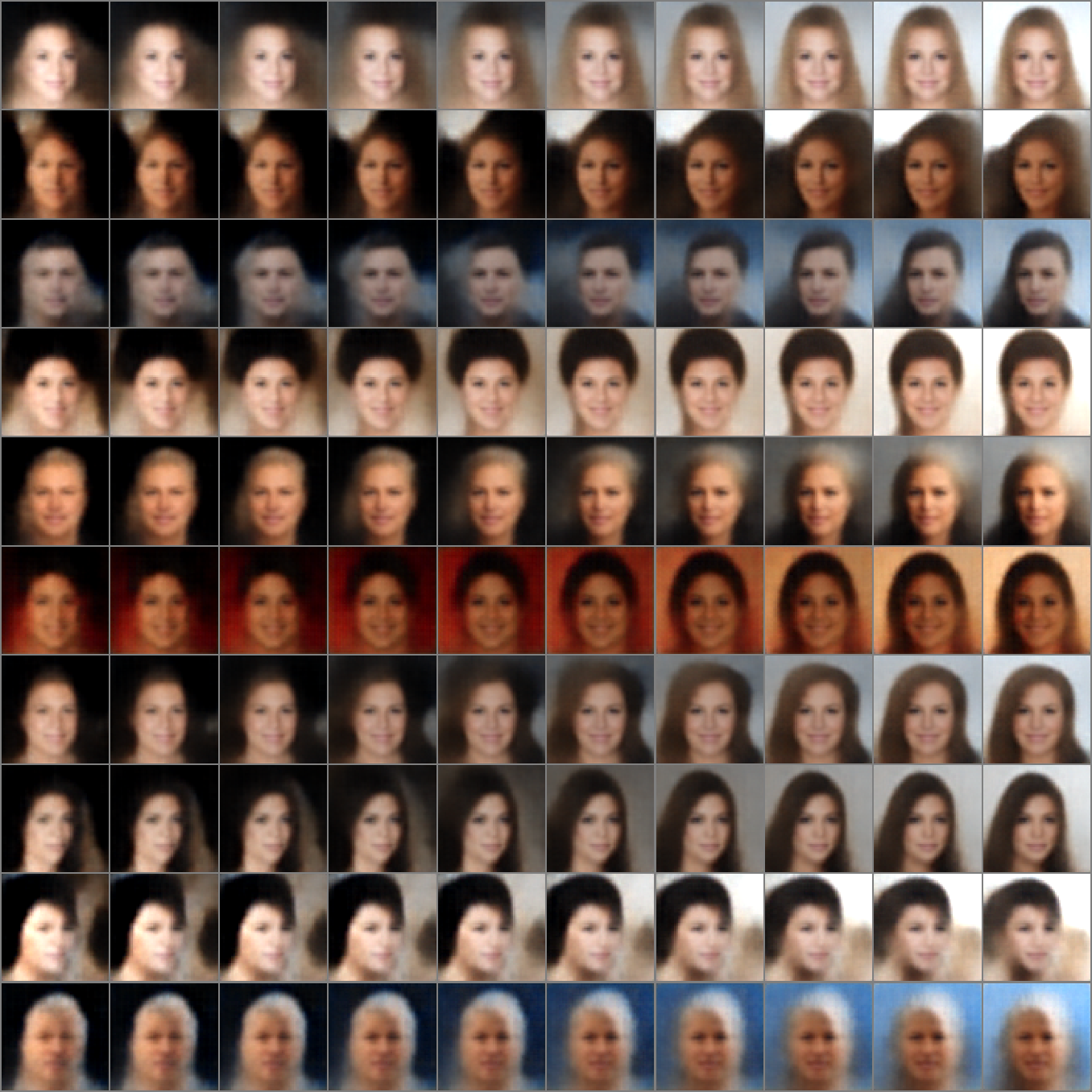}\label{celebAmog2beta40latent7}
          \end{minipage}
\hspace*{-0.003\textwidth}\begin{minipage}{0.32\textwidth}
             \centering
             %\vspace*{-0.1\textwidth}
                          {\tiny Factor 8: $I(\Xb;{Z_{\encb}}_{8})=0.8$, \\ $\mathop{\longrightarrow}\limits^{Smile}$}
        \includegraphics[width=\textwidth, clip=true, trim = 0 0 0 0]{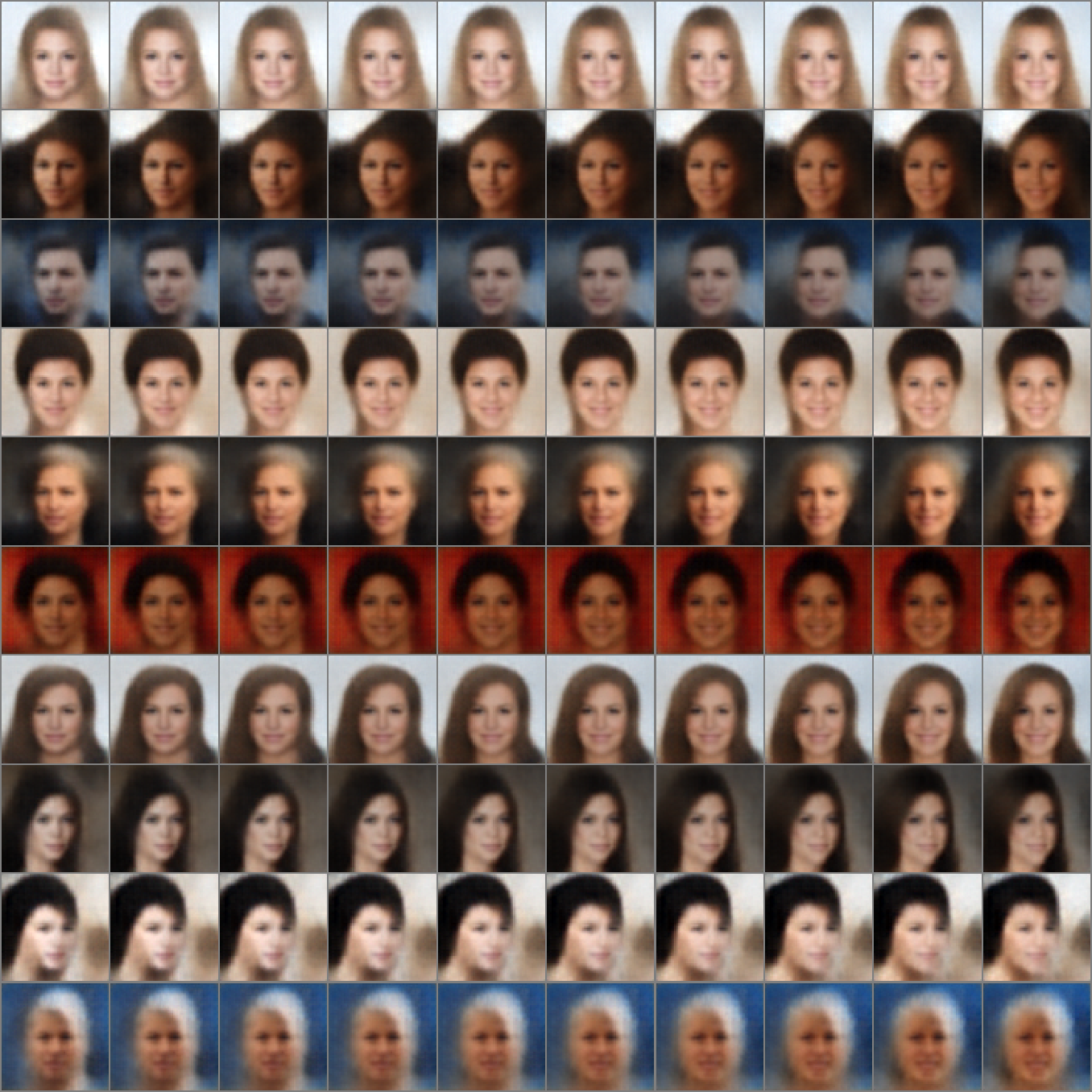}\label{celebAmog2beta40latent8}
          \end{minipage}
\hspace*{-0.003\textwidth}\begin{minipage}{0.32\textwidth}
             \centering
             %\vspace*{-0.1\textwidth}
              {\tiny Factor 12: $I(\Xb;{Z_{\encb}}_{12})=2.1$,\\
               $\mathop{\longrightarrow}\limits^{Turn \ Right}$}
        \includegraphics[width=\textwidth, clip=true, trim = 0 0 0 0]{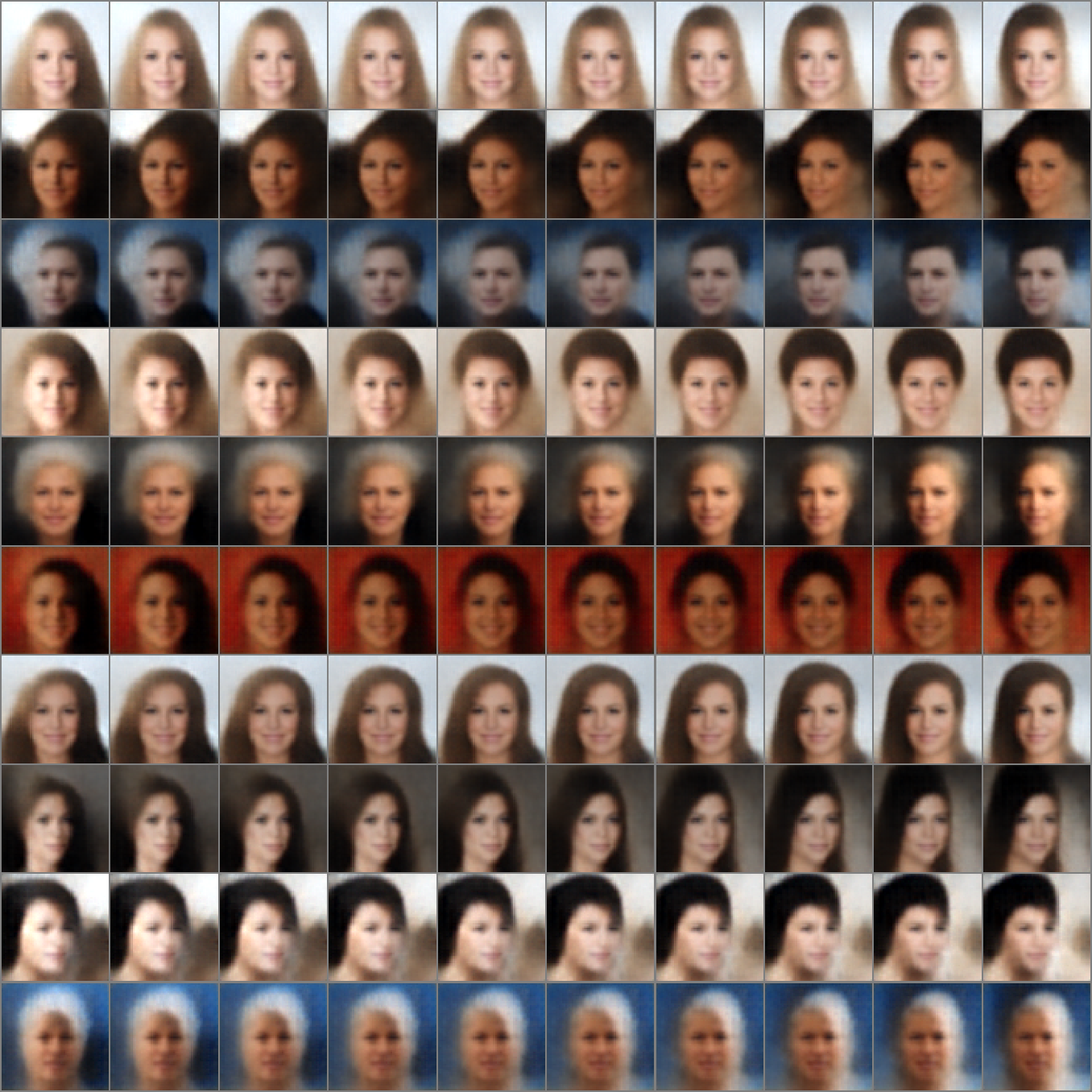}\label{celebAmog2beta40latent12}
          \end{minipage}
\caption{CelebA: Generating Factors Traversal of $\beta$(=40)-VAE. We present the first 3 influential factors determined by estimated mutual information. The whole influential factor traversals are listed in appendix~\ref{figure:FACE-BETA40-MOG2}.}
\label{figure:FACE}
\end{figure*}

In order to substantiate the validity of our mutual information estimator, we use it to automatically select influential factors with estimated $I(\Xb;{Z_{\encb}}_h)>0.5$ of CelebA shown in the Fig.(\ref{figure:FACE-BETA40-MOG2}) and many of them are possess the interpretable variants such  as background color, smile and face angle etc. This verifies that mutual information is an effective indicator to automatically determine the influential factors in the VAE setting.

\subsection{Generation Capability Test for Discovered Factors}
Estimated mutual information can instruct the latter generation task with few but influential factors. We select the different ratios of the top influential factors according to the quantity of the mutual information to generate the later image. The factors are sorted according to the values of its mutual information indicators and the other non-influential factors estimated by the indicator are constantly set to zero in the generating process.
\begin{figure}
  \centering
  \includegraphics[width=0.5\textwidth,clip=true, trim= 0 240 0 200 ]{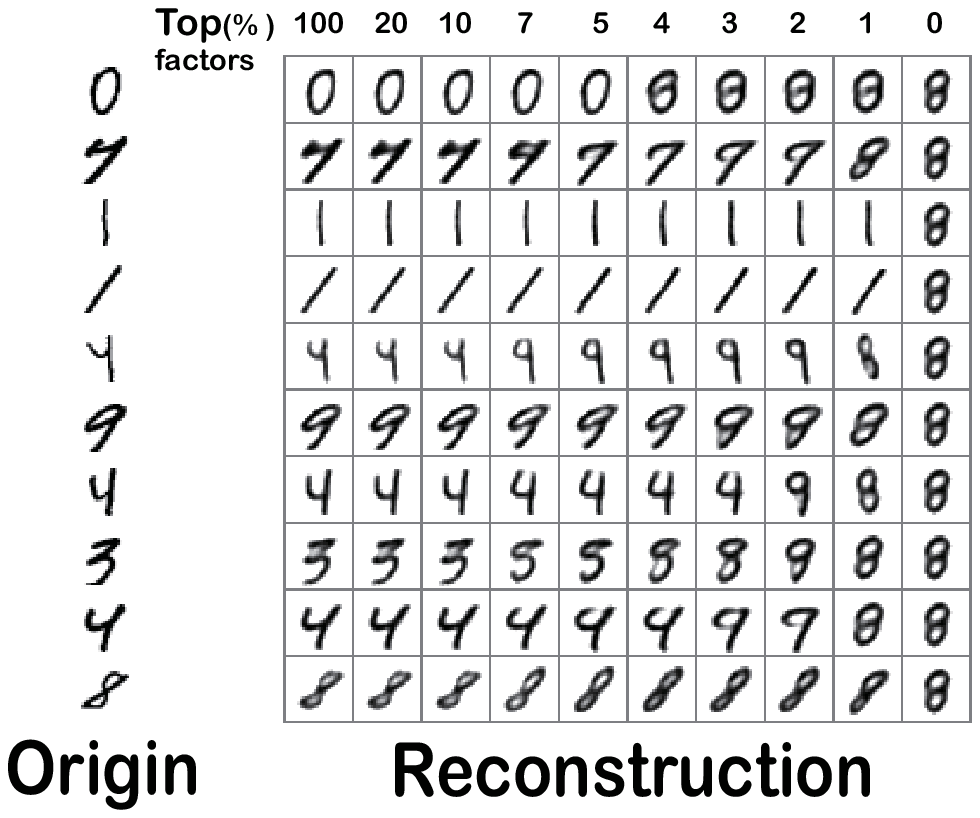}
  \caption{Generation plot with different ratio of factors.}\label{fig:ratio generation}
\end{figure}
According to Fig.(\ref{fig:ratio generation}), we can find that by on using $10\%$ of the top influential factors discovered by the proposed algorithm, the VAE model can still generate images almost similar to the one reconstructed by using whole factors.

Table~\ref{mutual information and reconstruction error plot} shows the detailed total information and the reconstruction error corresponding to the different ratio of factor. The top $10\%$ factors contain almost the whole information and therefore their reconstructions have the almost the same reconstruct error compared   to using all the factors. As suggested by the information and reconstruction relationship, the less information is contained in the used factors, the higher minimum reconstruction loss bound is raised.
\begin{table*}
\caption{mutual information and reconstruction error plot}\label{mutual information and reconstruction error plot}
\small
\begin{center}
\begin{tabular}{c c c c c c c c c c c}
  \hline
 Top(\%) factors in used & 100 &20 & 10 &7 & 5 & 4 & 3&2 &1 &0 \\ \hline
   $I(\Xb;\Zb_{\encb_{used}})$& 24.3&24.3	&24.3	&19.6	&16.5	&14.7	& 10.6	&8.4	&5.8&	0  \\ \hline
   mean square error& 5.6&5.6	&5.6	& 	13.4	&15.0&	18.9&	27.6	&31.3	&44.4&	71.7
  \\ \hline
\end{tabular}
\end{center}\normalsize
\end{table*}

\subsection{Classification Capability Test by Discovered Factors}
Estimated mutual information  can instruct the latter classification task with few but influential factors. We select the different ratio of the top influential factors according to the quantity of the mutual information to predict emotions. The factors are sorted according to its mutual information and the estimated non-influential factors are constantly set to zero in the prediction procedure.

\begin{table*}
\caption{Mutual information and EEG-emotion classification with $\beta(=6)$-VAE}
\begin{center}\small
\begin{tabular}{c c c c c c c c c c c}
  \hline
 Top(\%) factors in used & 100 &50 & 10 &7 & 5 & 4 & 3&2 &1 &0 \\ \hline
   $I(\Xb;\Zb_{\encb_{used}})$& 53.8&53.5	&38.3	&28.0	&22.5	&19.6	& 13.5	&10.2	&7.0&	0  \\ \hline
   mean test accuracy& 0.53&0.52	&0.46	& 	0.32	&0.34&	0.36& 0.29&	0.29	&0.3	&0.23
  \\ \hline
\end{tabular}
\end{center}\normalsize
\label{MI-EEG-Classfication}
\end{table*}

According to Table~\ref{MI-EEG-Classfication}, by only using half of the factors, the model still possesses the similar prediction accuracy. Besides, the estimated mutual information on the other side also helps us to determine the several variants which are relevant with the emotion classification as shown in the following Fig.(\ref{fig:3-emotion relevant factors discovery}).

\begin{figure*}
\centering
\hspace*{-0.003\textwidth}\begin{minipage}{0.32\textwidth}
             \centering
             %\vspace*{-0.1\textwidth}
                          {\tiny Factor 1: $I(\Xb;{Z_{\encb}}_1)=2.1$,\\ C4-T8-P4-P8 \ Block,\\ $\mathop{\longrightarrow}\limits^{ \ Red \ Turn \ Green}$ }
        \includegraphics[width=\textwidth, clip=true, trim = 0 0 0 0]{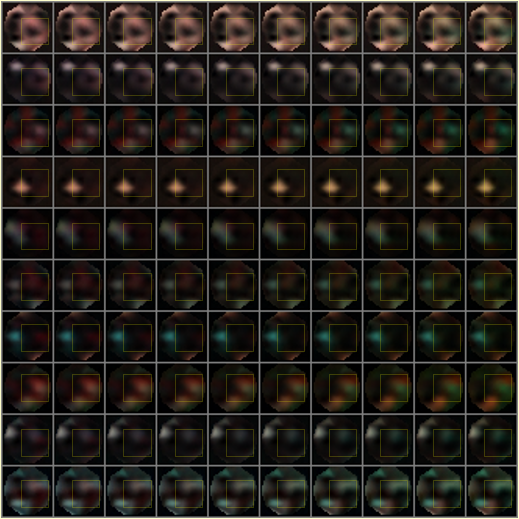}\label{EEGbeta6latent1}
          \end{minipage}
\hspace*{-0.003\textwidth}\begin{minipage}{0.32\textwidth}
             \centering
             %\vspace*{-0.1\textwidth}
                          {\tiny Factor 26: $I(\Xb;{Z_{\encb}}_{104})=3.2$,\\
PO3-O1-Oz \& Fp1-Fp2-AF4-AF3\  Block, \\$\mathop{\longrightarrow}\limits^{\ Turn \ Dark}$ }
        \includegraphics[width=\textwidth, clip=true, trim = 0 0 0 0]{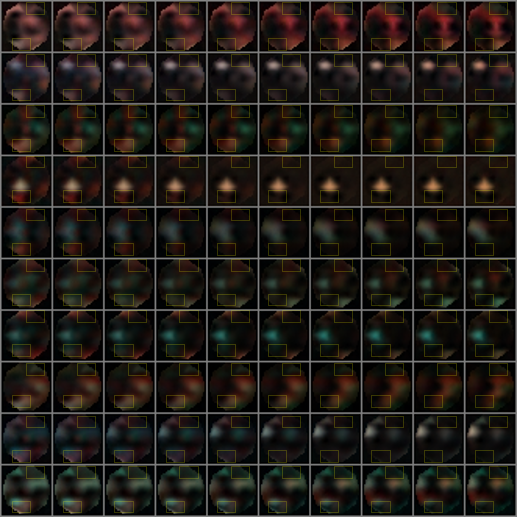}\label{EEGbeta6latent104}
          \end{minipage}
\hspace*{-0.003\textwidth}\begin{minipage}{0.32\textwidth}
             \centering
             %\vspace*{-0.1\textwidth}
{\tiny Factor 31: $I(\Xb;{Z_{\encb}}_{113})=3.7$,\\ PO3-O1-Oz \ Block, \\
                           $\mathop{\longrightarrow}\limits^{ Turn \ Green}$}
        \includegraphics[width=\textwidth, clip=true, trim = 0 0 0 0]{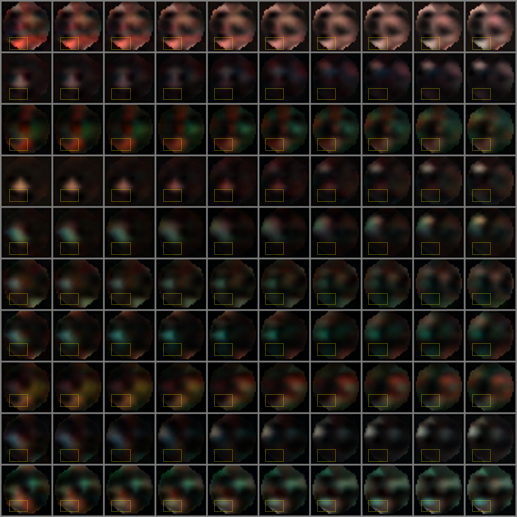}\label{EEGbeta6latent113}
          \end{minipage}
\caption{ Emotion relevant factors discovery. We present 3 influential factors determined by estimated mutual information. The whole influential factor traversals are listed in Fig.(\ref{fig:discover emotional relevant factor}) in appendix.}\label{fig:3-emotion relevant factors discovery}
\end{figure*}

\section{Conclusion}
This paper explains the necessity of using mutual information of the input data and each factor as the indicator to estimate the intrinsic influence of a factor to represent data in the VAE model. The mutual information reflects the absolute statistical dependence. The second term in the VAE objective and excess pre-set factors %over the data intrinsic dimension
 inclines to induce the mutual information sparsity and helps achieve influential, as well as ignored, factors in the VAE. We have also proved that the mutual information also involves in the lower bound of the mean square error of the reconstruction and of prediction error of the classification. We design a feasible algorithm to calculate the indicator for estimating the mutual information for all factors in the VAE and proves its consistency. The experiments show that  both the influential factors and non-influential factors can be automatically and effectively found. The interpretability of the  discovered  factors is substantiated intuitively, and the generalization and classification capability on these factors have also been verified. Specially, some variants relevant to  classification  are found. The experiments also inspire the idea that we can using a small amount of top influential factors for the latter data processing tasks including generation and classification by still keeping the performance of all factors, just similar to the dimensionality reduction capability as classical PCA, ICA and so on.

The VAE combined with mutual information  indicator helps automatically find vairiants and extract knowledge under the data and it can be applied to various variants of the VAE including  $\beta$-VAE (\cite{higgins2016beta}), FactorVAE (\cite{kim2018disentangling}), $\beta$-TCVAE (\cite{chen2018isolating}) and DIP-VAE (\cite{kumar2017variational}). It may be beneficial to extensive latter applications including blind source separation, interpretable feature learning, information bottleneck and data bias elimination. We will investigate these issues in our future research.

\section{Acknowledgments}
We would like to thank Zilu Ma and Tao Yu for discussing the information conservation theorems. We would like to thank Lingjiang Xie and Rui Qin for EEG data processing.

\section*{References}

\bibliography{reference}

\appendix
\onecolumn
\newpage

\section{Experiment Details\label{Experiment Details}}
\subsection{MNIST}
We split  7000 data points by ratio $[0.6:0.2:0.2]$ into training, validation, testing set.
The estimated mutual information and $q^*(\zb)$ are calculated on 10000 data points in the testing set. Seed images from the testing set are used to infer factor value and draw the traversal.

In traversal figures, each block corresponds to the traversal of a single factor over the $[-3, 3]$ range while keeping others fixed to their inferred (by $\beta$-VAE, VAE). Each row is generated with a different seed image.

The $\beta$ setting for $\beta$-VAE is enumerated from $[0.1,0.5,1,2:2:18]$.
\subsection{CelebA}
We split randomly roughly 200000 data points by ratio $[0.8:0.1:0.1]$ into training, validation (no use), testing set.

The estimated mutual information and $q^*(\zb)$ are calculated on 10000 data points in the testing set. Seed images from the testing set are used to infer factor value and draw the traversal.

In traversal figures, each block corresponds to the traversal of a single factor over the $[-3, 3]$ range while keeping others fixed to their inferred (by $\beta$-VAE, VAE). Each row is generated with a different seed image.

The $\beta$ setting for $\beta$-VAE is enumerated from $[1,30,40]$.
\subsection{DEAP}
DEAP is a well-known public multi-modalities (e.g. EEG, video, etc.) dataset proposed by \cite{koelstra2012DEAP}. The EEG signals are recorded from 32 channels by 32 participants watching 40 videos for 63 seconds each. The EEG data was preprocessed which down-sampling into 128Hz and band range 4-45 Hz. By the same transformation idea from \cite{bashivan2015learning}, we applied fast Fourier transform (FFT) on 1-second EEG signal and convert it to an image. In this experiment, alpha (8-13Hz), beta (13-30Hz) and gamma (30-45Hz) are extracted as the frequency band which represented the activities related to brain emotion emerging. The next step is similar as \cite{bashivan2015learning} work which mentioned in section II [PLEASE CHECK IT IN THE PAPER], by Azimuthal Equidistant Projection (AEP) and Clough-Tocher scheme resulting in three 32x32 size topographical activity maps corresponding to each frequency bands shown as RGB plot. The transformation work conduct the total of 1280 EEG videos where each has 63 frames. The two emotional dimensions are arousal and valence, which were labeled from the scale 1-9. For each of them, we applied 5 as the boundary for separating high and low level to generate 4 classes (e.g. high-arousal (HA), high-valence (HV), low-arousal (LA) and low-valence(LV)). In this paper we perform this 4-class classification task as same as the one in [baseline paper].

We split randomly roughly 1280 samples by ratio [0.8: 0.1: 0.1] into training, validation, testing set. $\beta$(= 6)-VAE is trained on each frame and LSTM was used to combine all the frames together for each video.

The estimated mutual information is calculated on 100*63 imagewise(100 videos) data points in the testing set. Seed images from the testing set are used to infer factor value and draw the traversal.

In traversal figures, each block corresponds to the traversal of a single factor over the $[-3, 3]$ range while keeping others fixed to their inferred (by $\beta$-VAE). Each row is generated with a different seed image.

%\begin{figure}[t]
%  \centering
%  \includegraphics[width=\textwidth]{figure/generatortraversal/EEGbeta6/Original.png}
%  \caption{Test set unfolded transformed images}
%   \label{fig:DEAP Test set unfolded transformed image}
%\end{figure}

\subsection{Network Structure}
\begin{center}
\small
  \begin{tabu}{c|c|c|c}
  \hline
  \textbf{Dataset} & \textbf{Optimiser}& \multicolumn{2}{c}{\textbf{Architecture}} \\ \hline
  \multirow{6}{*}{MNIST}& Adam & Input & 28x28x1  \\& $1e-3$ & Encoder & Conv 32x4x4,32x4x4 (stride 2).  \\ & & &FC 256. ReLU activation. \\&Epoch 200  & Latents &  128 \\&  & Decoder & FC 256. Linear. Deconv reverse of encoder. \\ & & &ReLU activation. Gaussian.  \\ \hline

  \multirow{6}{*}{CelebA}& Adam & Input & 64x64x3  \\& $1e-4$ & Encoder & Conv 32x4x4,32x4x4,64x4x4,64x4x4 (stride 2).  \\ & & &FC 256. ReLU activation. \\&Epoch 20  & Latents &  128/32 \\&  & Decoder & FC 256. Linear. Deconv reverse of encoder.  \\ & & & ReLU activation. Mixture of 2-Gaussian.\\ \hline
    \multirow{6}{*}{DEAP}& Adam & Input & 32x32x3  \\& $1e-4$ & Encoder & Conv 32x4x4,32x4x4,64x4x4,64x4x4 (stride 2).  \\ & & &FC 256. ReLU activation. \\&Epoch 300  & Latents &  128/32 \\&  & Decoder & FC 256. Linear. Deconv reverse of encoder.  \\ & &  & ReLU activation. Gaussian.\\
        &  & Input & 63x128  \\&  &  Recurrent & LSTM dim128. Time-Step 63. \\ & &  Predictor &  FC 4. ReLU activation.\\ \hline
%    \multirow{6}{*}{Extended Yale Face B}& Adam & Input & 192x168x1  \\& $1e-4$ & Encoder & Conv 32x4x4,32x4x4,64x4x4,64x4x4 (stride 2).\\ &Epoch 2002 & &FC 256. ReLU activation.  \\&  & Latents &  128 \\&  & Decoder & FC 256. Linear. Deconv reverse of encoder. \\ & & &  ReLU activation. \\ \hline
%      \multirow{6}{*}{Extended Yale Face B}& Adam & Input & 192x168x1  \\& $1e-4$ & Encoder & Conv 32x4x4,32x4x4,64x4x4,64x4x4 (stride 2).\\ &Epoch 1460 & &FC 256. ReLU activation.  \\&  & Latents &  128 \\&  & Decoder & FC 256. Linear. Deconv reverse of encoder. \\ (Network Parameterized Noise)& & &  ReLU activation. \\ \hline
\end{tabu}
\end{center}
\subsection{Experiment Plot}
In the following subsection, we present the influential factor ($I(\Xb;{Z_{\encb}}_{h})>0.7$) traversals, mutual information and variance plot of different data sets.

\begin{figure}
\centering
\foreach \t in {7,8,
12,13,19,26,28,%29,
31,37,40,45,48,
56,63,%65,
73,77,81,82,88,90,96,102,110,118,120}{
                \hspace*{-0.03\textwidth}  \begin{minipage}{0.21\textwidth}
             \centering
             %\vspace*{-0.1\textwidth}
                 {\tiny Factor \t}
          \includegraphics[width=0.9\textwidth, clip=true, trim = 0 0 0 0]{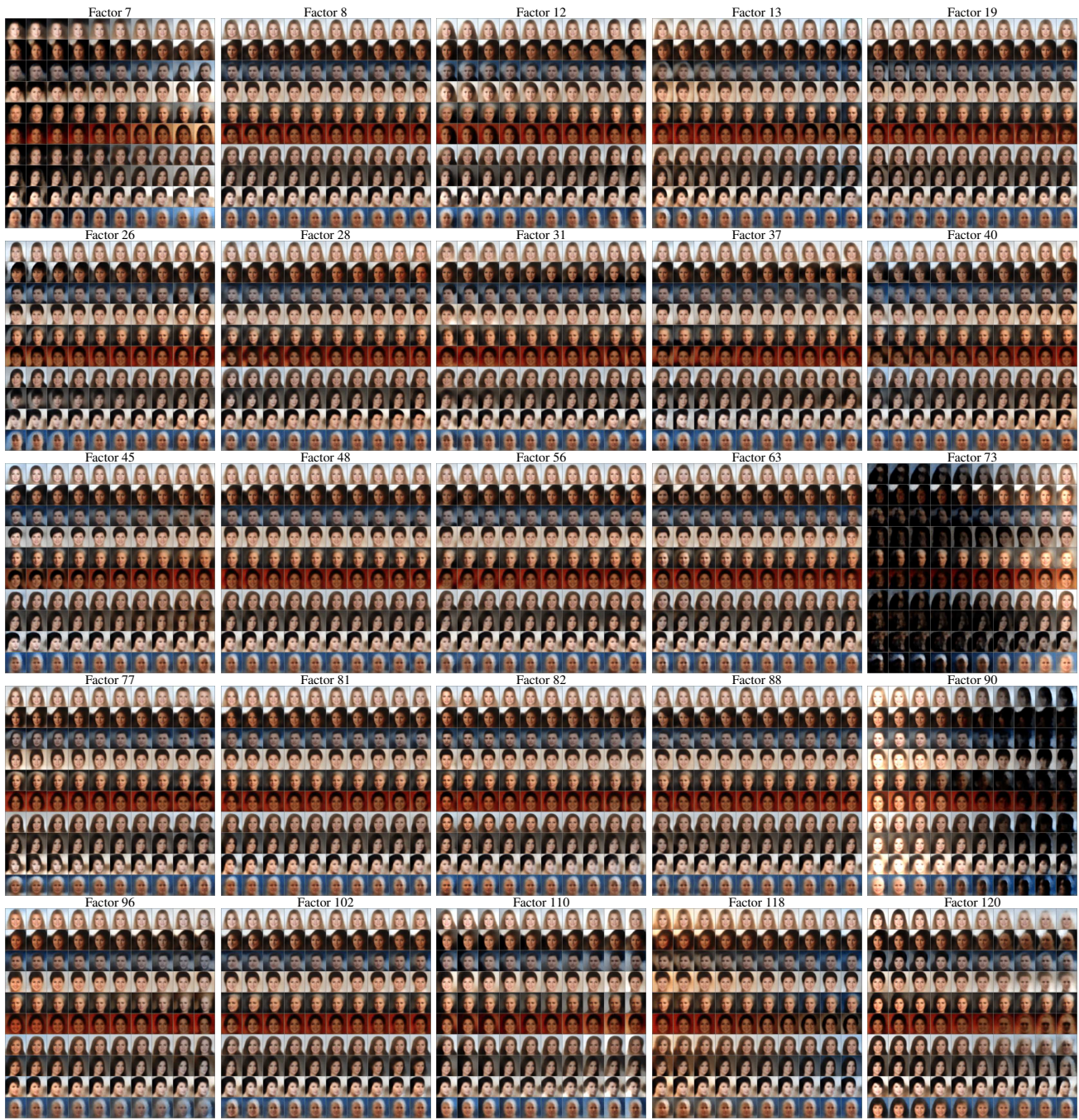}\label{fig:celebAmog2beta40latent\t}
          \end{minipage}
}
 \includegraphics[width=\textwidth]{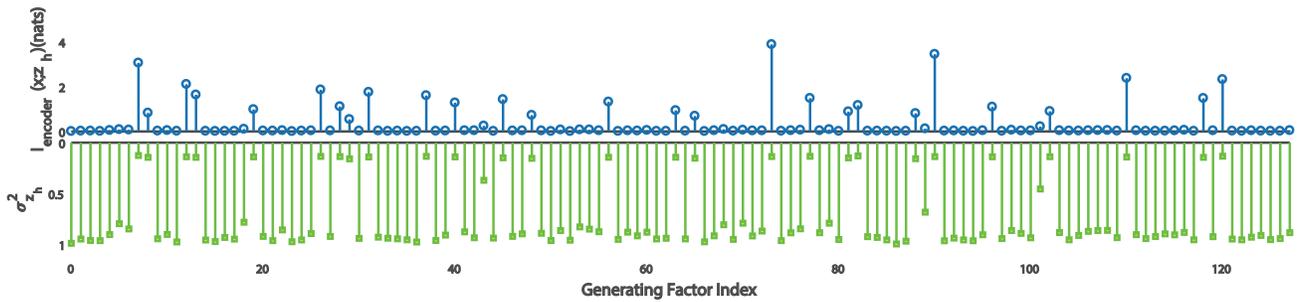}
\caption{\textbf{Mutual Information Sparsity in CelebA}: Generating Factors Traversal of $\beta$(=40)-VAE}
\label{figure:FACE-BETA40-MOG2}
\end{figure}

\begin{figure}
\centering
\foreach \t in {7,22,23,26,30,36,40,47,64,81,91,103,104}{
\begin{minipage}{0.19\textwidth}
              \centering
                  {\tiny \textbf{Factor} \t}
          \includegraphics[width=\textwidth]{figure/generatortraversal/mnistgamma10/latent\t.png}\label{MNISTbetalatent\t}
          \end{minipage}}
          \includegraphics[width=\textwidth]{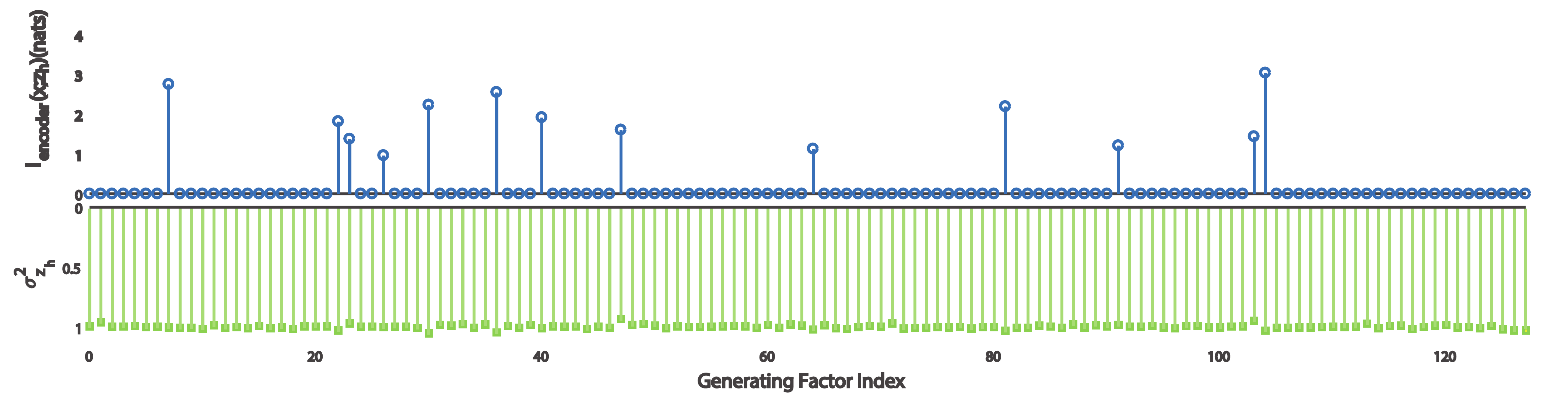}
\caption{\textbf{Mutual Information Sparsity in MNIST}: Generating Factor Traversal of $\beta$(=10)-VAE \label{fig:traservel of beta10MNIST}}
\end{figure}

\begin{figure}
\centering
\foreach \t in {1,26,31,36,40,59,60,61,64,65,69,83,86,88,100,104,113,114,116,117}{
\begin{minipage}{0.23\textwidth}
              \centering
                     {\tiny \textbf{Factor} \t   }
          \includegraphics[width=\textwidth, clip=true, trim = 0 0 0 0]{figure/generatortraversal/EEGbeta6/latent\t.png}\label{DEAPlatent\t}
          \end{minipage}    }
\includegraphics[width=\textwidth]{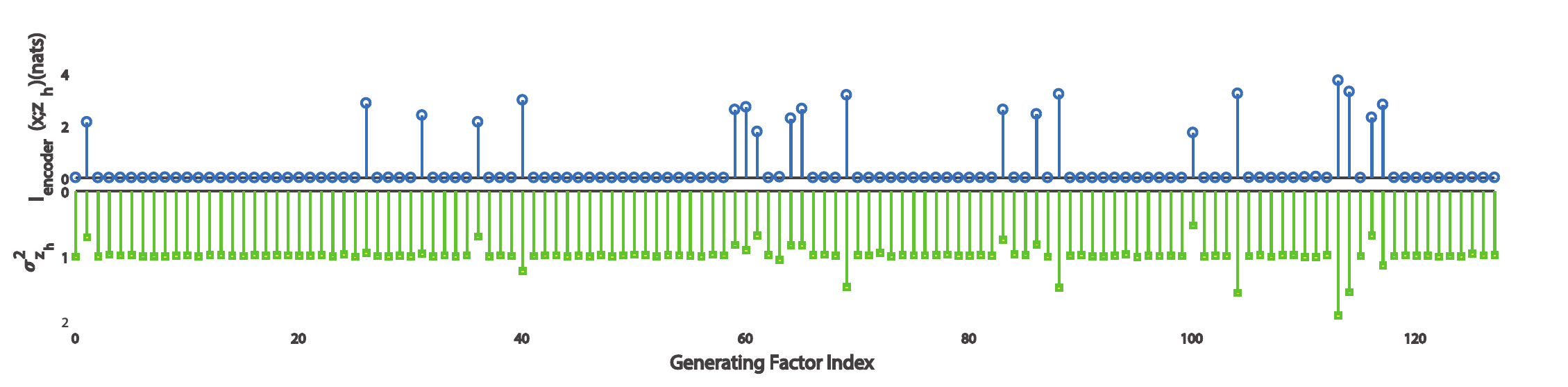}
\caption{\textbf{Mutual Information Sparsity in DEAP}: Generating Factor Traversal of $\beta$(=6)-VAE \label{fig:discover emotional relevant factor}}
\end{figure}

\end{document}